\documentclass[9.5pt,journal,compsoc]{IEEEtran}

\usepackage{bbm, dsfont} 
\usepackage{amssymb}
\usepackage[dvipsnames, table]{xcolor}
\usepackage{stmaryrd}
\usepackage{algorithm}
\usepackage{multirow}
\usepackage{capt-of}
\usepackage{hyperref}
\usepackage{afterpage} 
\usepackage{array}
\usepackage{eqparbox}
\usepackage{algpseudocode}

\usepackage{array}
\newcolumntype{P}[1]{>{\centering\arraybackslash}p{#1}} 
\definecolor{Gray}{gray}{0.85}

\usepackage{pifont}% http://ctan.org/pkg/pifont
\newcommand{\cmark}{\textcolor{OliveGreen}{\ding{51}}}%
\newcommand{\xmark}{\textcolor{red}{\ding{55}}}%

\usepackage{amsthm}

\newtheorem{lemma}{Lemma}

\usepackage{mathtools}

\definecolor{darkred}{RGB}{128,0,0}

\definecolor{flashyred}{RGB}{255,0,0}

\newcommand{\ourtitle}{Non-Uniform Post-Training Quantization via Power Exponent Search}
\ifCLASSOPTIONcompsoc
  \usepackage[nocompress]{cite}
\else
  \usepackage{cite}
\fi

\ifCLASSINFOpdf
  
\else
  
\fi

\begin{document}

\title{NUPES : \ourtitle}

\author{Edouard~Yvinec,
        Arnaud~Dapogny
        and~Kevin~Bailly
\IEEEcompsocitemizethanks{\IEEEcompsocthanksitem E. Yvinec is a PhD student at Datakalab, 114 boulevard Malesherbes 75017 Paris and Sorbonne Université, CNRS, Institut des Systèmes Intelligents et de Robotique, ISIR, F-75005 Paris, France.\\
\IEEEcompsocthanksitem A. Dapogny is a ML researcher at Datakalab.
\IEEEcompsocthanksitem K. Bailly is the Head of research at Datakalab and associate professor at Sorbonne Université.}}% <-this % stops an unwanted space

% The paper headers
\markboth{pre-print version}%
% \markboth{IEEE TRANSACTIONS ON PATTERN ANALYSIS AND MACHINE INTELLIGENCE}%
{Yvinec \MakeLowercase{\textit{et al.}}: \ourtitle}
\IEEEtitleabstractindextext{%
\begin{abstract}
Deep neural network (DNN) deployment has been confined to larger hardware devices due to their expensive computational requirements. This challenge has recently reached another scale with the emergence of large language models (LLMs). In order to reduce both their memory footprint and latency, a promising technique is quantization. It consists in converting floating point representations to low bit-width fixed point representations, usually by assuming a uniform mapping onto a regular grid. This process, referred to in the literature as uniform quantization, may however be ill-suited as most DNN weights and activations follow a bell-shaped distribution. This is even worse on LLMs whose weight distributions are known to exhibit large, high impact, outlier values. In this work, we propose an improvement over the most commonly adopted way to tackle this limitation in deep learning models quantization, namely, non-uniform quantization. NUPES leverages automorphisms to preserve the scalar multiplications. Such transformations are derived from power functions. However, the optimization of the exponent parameter and weight values remains a challenging and novel problem which could not be solved with previous post training optimization techniques which only learn to round up or down weight values in order to preserve the predictive function. We circumvent this limitation with a new paradigm: learning new quantized weights over the entire quantized space. Similarly, we enable the optimization of the power exponent, \textit{i.e.} the optimization of the quantization operator itself during training by alleviating all the numerical instabilities. 
The resulting predictive function is compatible with integer-only low-bit inference. We show the ability of the method to achieve state-of-the-art compression rates in both, data-free and data-driven configurations. Our empirical benchmarks highlight the ability of NUPES to circumvent the limitations of previous post-training quantization techniques on transformers and large language models in particular.
\end{abstract}

}

\maketitle

\IEEEdisplaynontitleabstractindextext
\IEEEpeerreviewmaketitle

\IEEEraisesectionheading{\section{Introduction}\label{sec:introduction}}

%%%%%%%%%%%%%%%%%%%%%%%%%%%%%%%%%%%%%%%%%%%%%%%%%%%%%%%%%%%%%%%%%%%%%%%%%%%%%%%%%%%%%
%
%
%                                    Introduction
%
%
%%%%%%%%%%%%%%%%%%%%%%%%%%%%%%%%%%%%%%%%%%%%%%%%%%%%%%%%%%%%%%%%%%%%%%%%%%%%%%%%%%%%%
\IEEEPARstart{T}{he} need for effective deep neural networks (DNNs) compression and acceleration techniques has grown as the computational requirements of DNNs increased. In particular, the recent surge of large language models (LLMs \cite{zhang2022opt}), characterized by their billions of parameters, has brought their specificities to the challenge of efficient inference.
To address the limitations on DNN deployment, quantization has become one of the crucial steps to DNNs deployment \cite{cheng2017survey}, especially on edge and low power devices. In its most general formulation, quantization consists in the conversion of large floating point representations to low-bit fixed point representations. As a result, the quantized model has a lower memory footprint as well as an improved latency. 

\begin{figure}
    \centering
    \includegraphics[width = \linewidth]{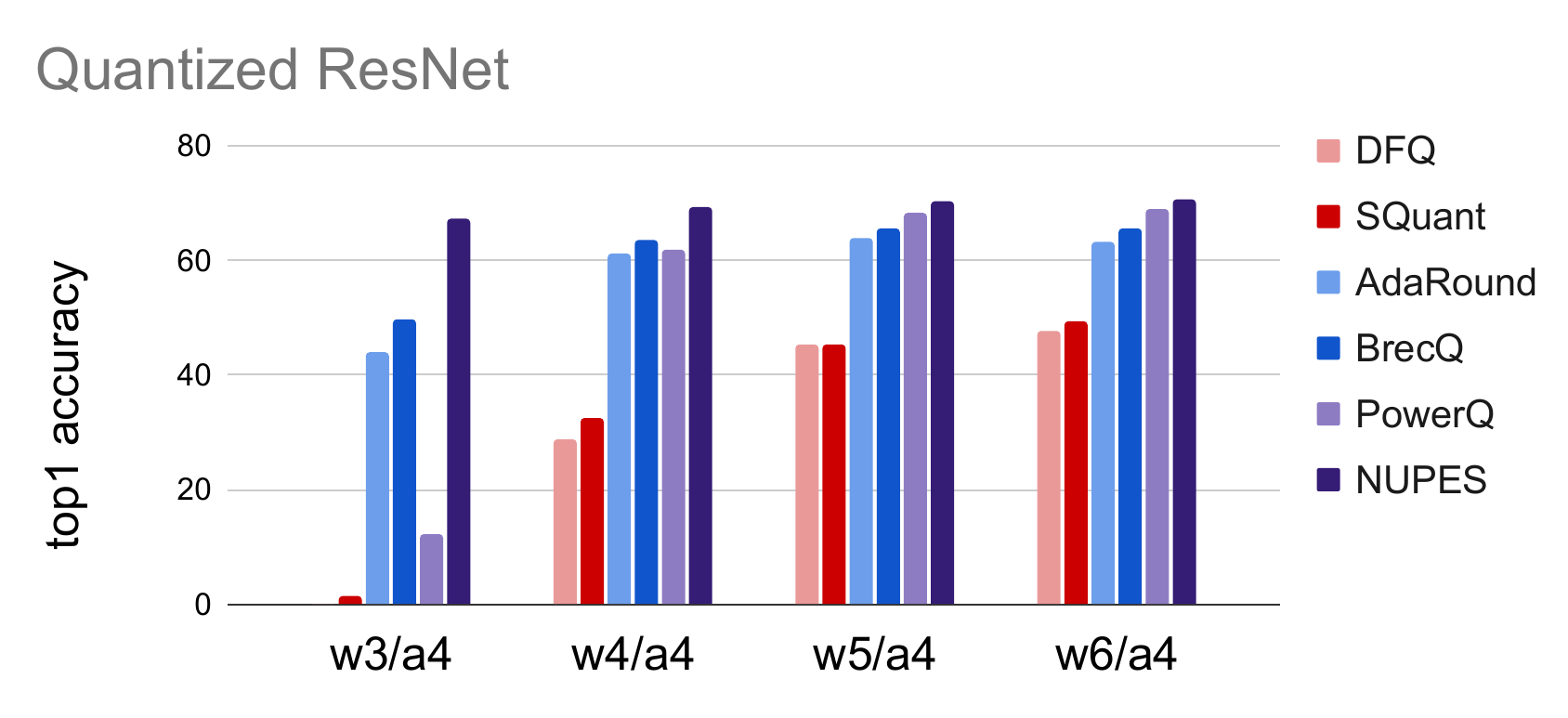}
    \includegraphics[width = \linewidth]{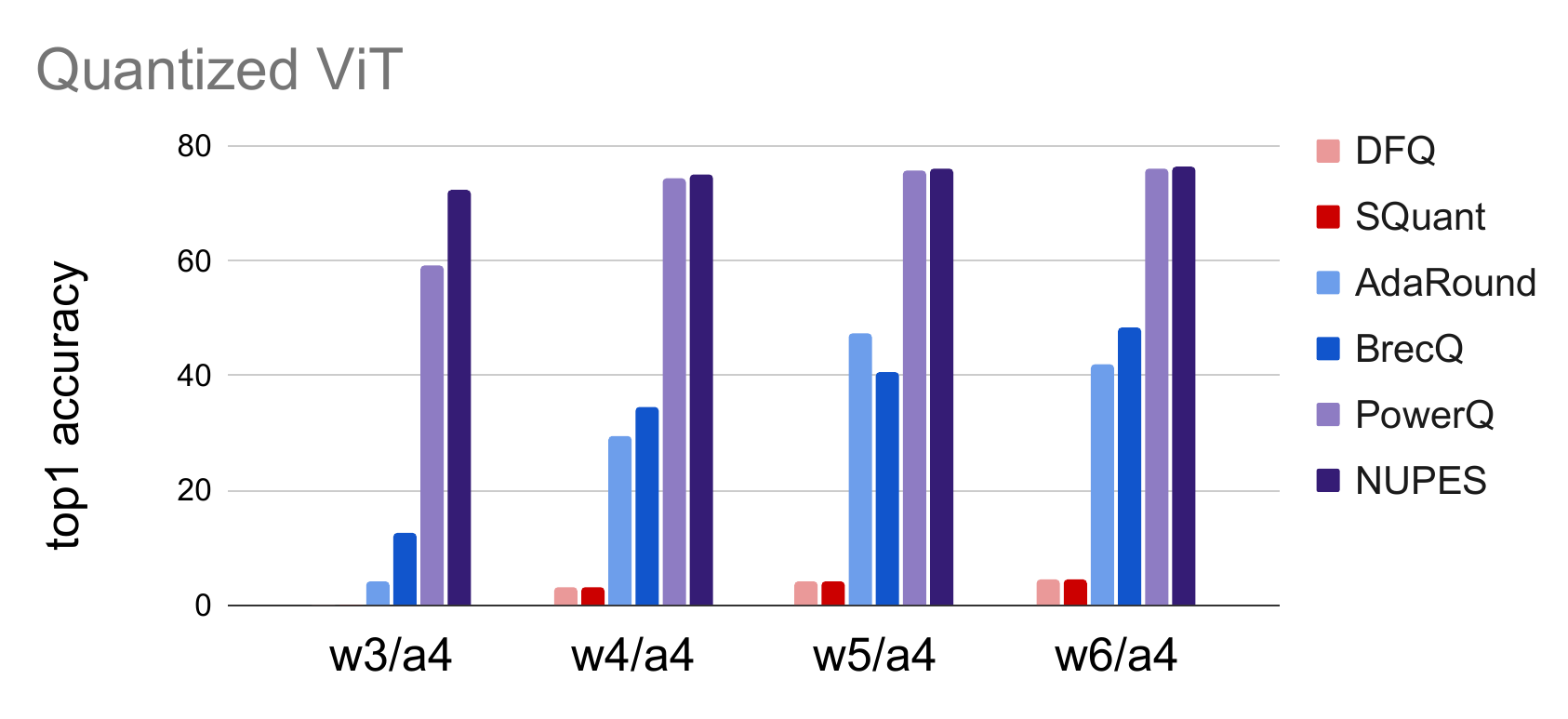}
    \caption{NUPES unlocks near full-precision accuracy in W4/A4 across ConvNets and Transformers. We can observe that while other methods are either more effective on ConvNets or Transformers, NUPES achieves state-of-the-art results on both use-cases.}
    \label{fig:main_figure}
\end{figure}

As defined by Nagel \textit{et al.} \cite{nagel2019data}, quantization techniques are classified based on their data requirements. In its most practical setup from the user perspective, but most challenging in terms of compression, quantization only leverages the weight values of a pre-trained model. Data-free quantization techniques \cite{nagel2019data,zhao2019improving,squant2022,yvinec2022spiq} are motivated by the growing concerns regarding data privacy as well as their scalability with respect to the model size. On the other end of the spectrum, quantization aware training techniques \cite{courbariaux2015binaryconnect,courbariaux2016binarized,zhang2022pokebnn} enable the quantization of ConvNets into binary representations at the cost of a full retraining burdened by extra operations. In order to offer more trade-offs in terms practical usage, gradient-based post training quantization (GPTQ) \cite{nagel2020up,li2021brecq} was introduced and consists in learning the quantized weight values on a small calibration set. In practice, it is usually limited to the optimization of the rounding operation (rounding up or down).

As opposed to most quantization methods \cite{nagel2019data,zhang2022pokebnn,li2021brecq,yvinec2022spiq} which, for the sake of practicality, map floating point values to an evenly spread, discrete space, non-uniform quantization achieves a tighter fit to the target distribution by using a non-uniform spread quantized target space. However such methods \cite{banner2019post,miyashita2016convolutional} require custom implementations that significantly shift away from uniform quantization.
In our previous work, PowerQuant \cite{yvinec2023powerquant}, we proposed a novel non-uniform quantization technique which both massively improved the performance of data-free quantization and tackled the requirement of the aforementioned methods for custom operations.
As we wanted to map multiplications to multiplications,  we searched for the best possible quantization operator that preserves the mathematical operations among automorphisms of $(\mathbb{R}^*_+,\times)$. In practice, we showed that it boiled down to the optimization of the exponent parameter of a power function with respect to the error introduced by quantization. 

In this work, we propose NUPES, a novel post-training non-uniform quantization to address the limitations of power quantization and study its performance on large language models. First, we adapt PowerQuant to the GPTQ framework \cite{nagel2020up}. While other GPTQ methods optimize a continuous value $\epsilon \in [0;1]$ in the quantized space as a way to learn the rounding operation, this cannot be applied to PowerQuant. The challenge arises from the disparity in ranges, for values close to zero, a shift of $1$ in the quantized space is negligible in the full precision space, while a change of $1$ for larger values leads to a substantial change in the predictive function. To circumvent this limitation, we draw inspiration from quantization aware training in order to design a novel, robust way to learn $\epsilon \in \mathbb{Z}$.
This new method, dubbed NUPES, enables us to reach new levels of compression while maintaining the accuracy of the original method as shown on Figure \ref{fig:main_figure}. While PowerQuant already outperformed GPTQ methods on transformers, NUPES further improves it and extends this ability to convolutional neural networks.
Second, we propose to alleviate a limitation of PowerQuant: the search for the power exponent. In the original method, we proposed to optimize a single value for the whole network based on the weights. In NUPES, we propose to learn the exponent parameter per-layer through stochastic gradient descent \cite{ruder2016overview} during the GPTQ process which results in more flexibility and paves the way towards the integration of PowerQuant in a full quantization aware training. Third, we demonstrate the particular efficiency of PowerQuant regarding transformers quantization and outliers handling. As transformers grow in size, they become more challenging to quantize due to the emergence of outliers \cite{dettmers2022llm}. These values stretch the distribution range, which leads to a significant part of the uniform target space being wasted. Consequently, we provide an extensive study on the ability of PowerQuant to efficiently address this pitfall on large transformers and large language models (LLMs).

\section{Related Work}

\subsection{Quantization}
In this section, we describe the current state-of-the-art in artificial neural network quantization. As our goal is both memory footprint and latency reduction, we will omit approaches that are designed for weights storage (on disk not on RAM) only \cite{chen2015compressing,gong2014compressing,zhou2017incremental}.

\subsection{Non-Uniform Quantization}
In general, mapping a non-uniform continuous distribution to a discrete evenly spread space linearly is unlikely to achieve the best usage of the target quantized space. To alleviate this limitation, non-uniform quantization has been studied \cite{banner2019post,hubara2016binarized,jeon2020biqgemm}. Among these techniques, code-base approaches \cite{banner2019post,jeon2020biqgemm} usually come at the price of complex implementations at runtime. The remaining methods alter the quantization operator \cite{miyashita2016convolutional,zhou2017incremental}. In particular, Zhang \textit{et al.} \cite{zhang2021training} proposed a logarithm-based quantization mechanism. Hence, the values in the quantized space no longer represent scalar values for a multiplication but rather an exponent for a bit-shift which is applied to the uniformly quantized inputs. Consequently, log-quant methods require a change in the nature of the supported mathematical operations and are bound to a low support until their performance justify investing in new inference instruction sets. In this work, we propose an alternative to these two non-uniform paradigms, namely NUPES, which consists in the search for quantized weights and a quantization operator defined by a power exponent which maps scalar multiplications to scalar multiplications. Hence, by preserving the mathematical operations, NUPES is easier to leverage. We adapt the method for gradient-based optimizations, thus tremendously improving its performance. 

\subsection{Gradient-based Post-Training Quantization}
The surge of large language models and the struggle to deploy them efficiently has sparked a novel interest for post-training quantization. For instance, a quantization aware training \cite{zhang2022pokebnn} on a small model such as ResNet 50 \cite{he2016deep} takes 7 days on 8 gpus A100. In this context, Nagel \textit{et al.} \cite{nagel2020up} pioneered with the so-called AdaRound method, that consists in two steps. First, quantizing the weights and activations in a naive fashion \cite{krishnamoorthi2018quantizing} with a rounding step for the activations and a flooring step for the weights. Second, using a calibration set (subset from the training set) in order to learn the rounding $\epsilon \in [0;1]$ using stochastic gradient descent. This method is performed layer per layer and uses as ground truth the intermediate features of the original, full-precision, model. Since its publication, several works \cite{li2021brecq,wei2022qdrop,liu2023pd} aimed at improving the accuracy of the quantized model in the low bit setup. The most noticeable one, BrecQ \cite{li2021brecq}, groups layers during the optimization step based on the architecture. For instance, all the layers in a residual block \cite{he2016deep} or a transformer block \cite{vaswani2017attention} are optimized at once, thus reducing the duration of the process while improving the final accuracy. However, non-uniform quantization is not compatible with GPTQ methods in their current form. NUPES solves this problem and achieves state-of-the-art accuracy on most commonly used architectures by a) generalizing the rounding search and b) optimizing over the power exponent, leading to superior performance. In particular, LLMs and outliers.

\subsection{Quantization of Large Language Models}
The rise of large language models has brought its own share of challenges in terms of quantization. In particular, the presence of outliers among weight and activations \cite{dettmers2022llm}, has proven to lead to significant accuracy loss when performing previously introduced state-of-the-art quantization schemes such as DFQ \cite{nagel2019data} or SQuant \cite{squant2022}. This phenomenon is illustrated in Figure \ref{fig:outliers} where an outlier is defined as the number of standard deviations from the mean absolute weight value. We observe that LLMs have outliers that are over 17 standard deviations away from the mean (more than the number of values that can represented by an int4) while ResNet architectures are bounded by 8 standard deviations. Consequently, such neurons will see most of their values quantized to zero almost as if they were quantized in ternary values. Furthermore, due to the size of the models, even efficient quantization methods such as AdaRound \cite{nagel2020up} fail to efficiently scale and introduce a expensive processing time, taking a few days to quantize a mid-sized LLM. Most recent works \cite{dettmers2023qlora,frantar2023optq}, that tackle weight quantization of LLMs, all leverage a new quantization granularity, namely group-wise quantization, introduced in nuQmm \cite{park2022nuqmm}. In our experiments, we demonstrate that power quantization in general enables state-of-the-art performance on the challenging LLMs.

\section{Methodology}
Let's consider $F$, a trained neural network defined by its layers ${(f_l)}_{l\in\{1,...,L\}}$ of real valued weight tensors ${(W_l)}_{l\in\{1,...,L\}}$. Formally, a quantization operator $Q$ is a transformation that maps all the elements $x\in\mathbb{R}$ of any tensor $X$ to a quantized interval $[- 2^{b-1} ; 2^{b-1} -1] \cap \mathbb{Z}$. Hence, the computation of $Q$ involves a rounding operation that we note $\lfloor \cdot \rceil$. Furthermore, in order to cover the entire quantized space, the quantization process $Q$, almost systematically, involves a scaling step. Consequently, we define the baseline operator $Q$ as
\begin{equation}\label{eq:baseline_operator}
    Q(x) = \left\lfloor\frac{x}{ s(X) }\right\rceil \quad \text{with } s(X) = \frac{\max_{x\in X}\{|x|\}}{2^{b-1}-1}
\end{equation}
The $Q$ operator comes with two limitations. First, by definition any quantization operator introduces an information loss as $Q^{-1}(Q(X)) = s(X) \times Q(X) \neq X$. Worse, given that in practice $X$ will follow a non-uniform distribution, the use of a linear scaling transformation is likely to be sub-optimal. Intuitively, most values in $X$ will be closer to zero than the edge of their support in $\mathbb{R}$ while $Q$ gives as much precision to every piece of that support. This phenomenon is illustrated in Fig \ref{fig:intuition} (a). Second, by definition of the scaling term $s$, the quantization process is very sensitive to outliers in $X$. Any such outlier would either be clipped to fit the quantized space or stretch the support out, leading to many non-zero values being quantized to zero as illustrated in Fig \ref{fig:intuition} (b).
In our previous work, we introduced PowerQuant \cite{yvinec2023powerquant}, a novel quantization operator that tackles both of these issues at once. 

\begin{figure}
    \centering
    \includegraphics[width = \linewidth]{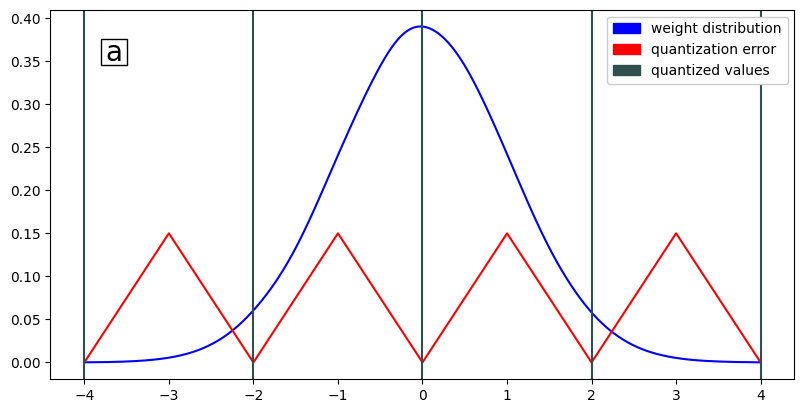}
    \includegraphics[width = \linewidth]{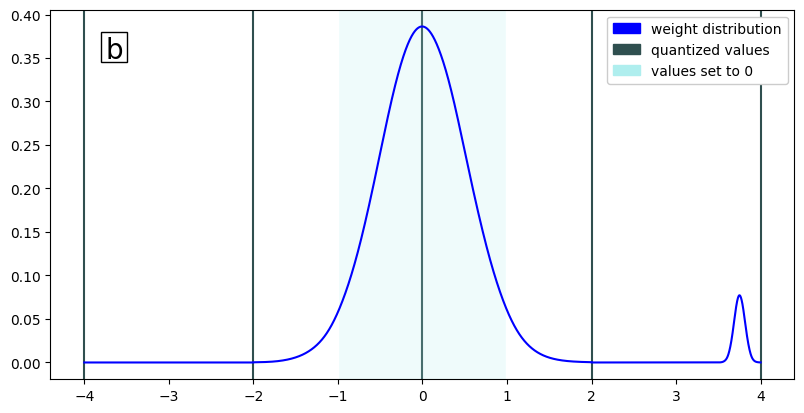}
    \caption{Illustration of the two pitfalls of naive quantization. First (a), the quantization error is uniformly distributed regardless of the prior weight distribution (Gaussian). In short, extreme values will have the same quantization error as redundant values despite their much rarer occurrences. Second (b), outliers stretch the distribution out and lead to zero quantization of most of the weight values.}
    \label{fig:intuition}
\end{figure}

\subsection{PowerQuant}\label{sec:powerquant_methodo}
We generalize the definition of a quantization operator $Q$ from equation \ref{eq:baseline_operator} through transformations $t$ of the tensor $X$ to quantize as
\begin{equation}\label{eq:general_operator}
    Q(x) = \left\lfloor\frac{t(x)}{ s(t(X)) }\right\rceil \quad \text{with } t:\mathbb{R}\rightarrow \mathbb{R}.
\end{equation}
In practice $t$ can be learned \cite{nagel2020up,li2021brecq} or based on some heuristic \cite{squant2022}. In PowerQuant, we define $t$ such that for any pair $(x,y)$ of positive real values we get $t(x)\times t(y) = t(x\times y)$. The set of such functions is the set of automorphisms of $(\mathbb{R}_+^*,\times)$. This set corresponds to the set of power functions $t:x\mapsto x^a$ for $a\in \mathbb{R}$ (see Appendix \ref{sec:appendix_automorphism_set}).  Consequently, the set $\mathcal{Q}$ of power operators is defined as
\begin{equation}\label{eq:power_operators}
    \mathcal{Q} = \left\{ Q_a : X \mapsto \left\lfloor \frac{\text{sign}(X)\times |X|^a}{s(\text{sign}(X)\times |X|^a)} \right\rceil \Big| a \in \mathbb{R} \right\}.
\end{equation}
In other words, for a given power value $a$, the transformation $t$ in equation \ref{eq:general_operator} is $t:x\mapsto \text{sign}(x)\times |x|^a$ where the $\text{sign}(\cdot)$ function generalizes the automorphisms of $\mathbb{R}_+^*$ to $\mathbb{R}$. In prior work, we proposed to search, in a data-free manner, for the quantization operator in $\mathcal{Q}$ that minimizes the quantization reconstruction error. Formally, this boiled down to solving the following minimization problem:
\begin{equation}\label{eq:minimization_problem}
    \min_{a\in\mathbb{R}} \left\{ \sum_{l=1}^L \left\| W_l - Q^{-1}\left(Q\left(W_l\right)\right) \right\|_2 \right\}.
\end{equation}
This problem is locally convex around its unique global solution (we recall the demonstration in Appendix \ref{sec:appendix_maths}). Based on these properties, Equation \ref{eq:minimization_problem} can be solved using the Nelder-Mead method \cite{nelder1965simplex}. Empirically, we observed that this problem was better solved using a single parameter shared for all layers and based only on the weight values. Furthermore, $a=0.5$ enabled near optimal performance in terms of accuracy while requiring the implementation of simpler power functions (square and square roots). Thus leading to a more straightforward integration at inference.

\begin{figure*}[!t]
    \centering
    \includegraphics[width = 0.24\linewidth]{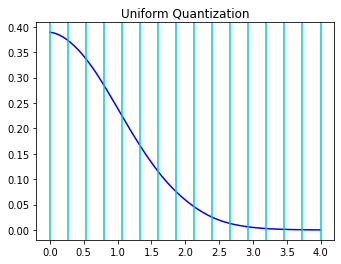}
    \includegraphics[width = 0.24\linewidth]{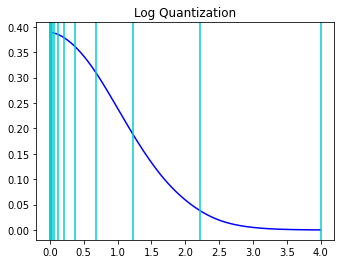}
    \includegraphics[width = 0.24\linewidth]{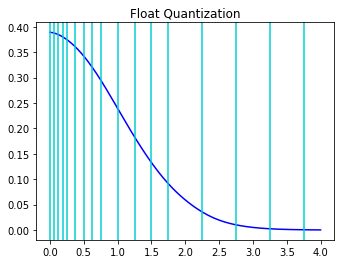}
    \includegraphics[width = 0.24\linewidth]{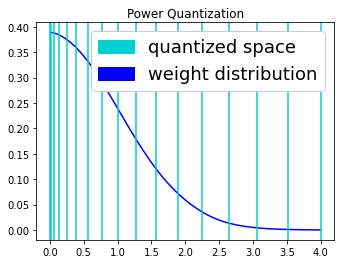}
    \caption{Illustration of the quantization in 4 bits of the positive part of a Gaussian distribution in uniform, logarithm, floating point and power quantization (from left to right). This highlights the balance between precision for very low bit values and larger values that is achieved with PowerQuant \textit{versus} other quantization formats.}
    \label{fig:power_comparison}
\end{figure*}

Intuitively, a value of $a=1$ corresponds to the naive uniform quantization as defined in equation \ref{eq:baseline_operator}. While using a value of $a<1$ puts more emphasis (and precision) to scalar values with a small magnitude. Such values of $a$ result in a strong similarity with floating representations and log-based quantization (see Fig \ref{fig:power_comparison}). In short, such representations are known for their strong ability to better fit to empirical distributions of weights and activations. For instance, the new floating format introduced in QLoRa \cite{dettmers2023qlora} converges to a more rigid version of PowerQuant. While QLoRa is optimal for Gaussian distributions, weight and activations do not strictly follow such prior. PowerQuant offers an empirically more accurate solution which we attribute to its greater flexibility (optimization of $a$) and smoother step size changes.

Nonetheless, PowerQuant has several limitations. First, weight values in power operators cannot be straightforwardly optimized in a GPTQ fashion. Intuitively, GPTQ methods only learn to round up or down which corresponds to the same constraint for all values in the quantized space. However, in non-uniform quantization, all values in the quantized space have significantly different impact in the full-precision space. This discrepancy leads to incoherent optimization and a degradation of the performance. To solve this issue, we propose a novel approach to GPTQ methods for non-uniform quantization. Second, the optimization of $a$ is shared across the network (one value for the model and not one value per-layer) and is limited to the weight values and does not account for the activations. Furthermore, the optimization of $a$, in its current form, has to be performed post-training. In the following sections, we describe how the proposed NUPES method addresses these two limitations.

\subsection{Gradient Based Optimization of the Weight}\label{sec:methodology_kernel}

GPTQ was introduced by Nagel \textit{et al.} \cite{nagel2020up} as a layer-wise self-distillation framework (rather than a full model optimization at once). For a given layer $l$, we assume that every upstream layer has already been optimized. We note $X_{\text{fp}}$ and $X_{\text{q}}$ the intermediate input features of layer $l$ from the full precision and quantized models respectively. Then the goal is to learn the rounding operation by minimizing 
\begin{equation}\label{eq:gptq}
    \min_{\epsilon} \left\| f_l(X_{\text{fp}},W_l) - f_l\left(X_q, \left\lfloor\frac{W_l}{ s(W_l) }\right\rfloor + \sigma(\epsilon) \right) \right\|_2^2.
\end{equation}
Intuitively, as $W_l$ is rounded down $\lfloor\cdot\rfloor$, we learn through stochastic gradient descent a real valued parameter $\epsilon$ initialized such that $ s(W_l) \left(\left\lfloor\frac{W_l}{ s(W_l) }\right\rfloor + \sigma(\epsilon)\right) = W_l$. At the end of the optimization process, $\epsilon$ defines whether we round $W_l$ up or down. In order to bound the influence of $\epsilon$ to [0;1], GPTQ methods apply a variant of sigmoid, called the rectified sigmoid:
\begin{equation}\label{eq:sigmoid}
    \sigma : \epsilon \mapsto \text{clip}\left(\frac{1}{1+e^{-\epsilon}} \times 1.2 - 0.1, 0, 1\right).
\end{equation}

During the optimisation process over the calibration set (usually consisting in $\approx 1k$ examples), a regularization term is added, which forces $\sigma(\epsilon)$ to converge to either zero or one. Consequently, the minimization loss $\mathcal{L}$ is
\begin{equation}\label{eq:gptq_loss}
\begin{aligned}
    \mathcal{L} &= \left\| f_l(X_{\text{fp}},W_l) - f_l\left(X_q, \left\lfloor\frac{W_l}{ s(W_l) }\right\rfloor + \sigma(\epsilon) \right) \right\|_2^2\\
    &+ \lambda \times \left( 1 - \left( 2 \times \sigma(\epsilon) - 1\right)^\beta\right)
\end{aligned}
\end{equation}
where $\lambda$ and $\beta$ are two hyper-parameters with their respective schedulers: the lambda parameter defines whether the rounding term of the loss is used or not, in practice, it is set to zero for the first $20\%$ of the optimization of each layer and then set to $0.01$. The parameter $\beta$ defines the steepness of the rounding loss (the larger $\beta$ the steeper the rounding). The schedule of $\beta$ has a great influence on the performance and will be discussed in the experiments. As a result, one needs to store in memory both $\epsilon$ and $\left\lfloor\frac{W_l}{ s(W_l) }\right\rfloor$ and compute the two loss terms. 
This slows down the optimization process and hinders its scalability especially on larger models. Furthermore, as illustrated in Figure \ref{fig:power_comparison}, PowerQuant quantized space has a very small step size for values near $0$ which implies that optimizing $\epsilon$ can only lead to marginal corrections. On the other hand, for larger values, a change in the rounding operation may lead to tremendous changes and thus introduce instability in the optimization process. In order to circumvent all of these limitations at once, we propose to leverage the differentiable soft quantization activation \cite{gong2019differentiable}
\begin{equation}\label{eq:dsq}
    \text{dsq}(\epsilon) = \frac{\text{tanh}\left(\beta \times \left(\epsilon - \frac{1}{2} - \lfloor\epsilon\rfloor\right)\right)}{2\text{tanh}\left(\frac{\beta}{2}\right)}  + \lfloor\epsilon\rfloor + \frac{1}{2}
\end{equation}
where $\beta$ is the hyper-parameter replacing $\beta$ in equation \ref{eq:gptq_loss} and eliminating the need for $\lambda$. Fig \ref{fig:dsq_graph} illustrates the influence of steepness parameter $\beta$ on the dsq function.
Consequently, our new, simplified objective and loss function become
\begin{equation}\label{eq:dsq_}
    \begin{cases}
        \min_{\epsilon} \left\| f_l(X_{\text{fp}},W_l) - f_l\left(X_q, \text{dsq}(\epsilon) \right) \right\|_2^2\\
        \mathcal{L} = \left\| f_l(X_{\text{fp}},W_l) - f_l\left(X_q, \text{dsq}(\epsilon) \right) \right\|_2^2\\
    \end{cases}
\end{equation}
and is optimized from the initial value $\epsilon = \frac{W_l}{ s(W_l)}$ which learns the new quantized values of $W_l$ as a single tensor, thus enabling values to be shifted by more than $1$ (we no longer only learn to round up or down). Consequently, we can learn larger modifications of values near $0$ and compensate for modifications of larger values. Furthermore, as we only use $\epsilon$ and no longer need the floored weight values, NUPES only requires half the memory footprint (for the model) and only computes one loss term. The memory footprint reduction is particularly important for transformer architectures, for which weight representations are typically more memory-consuming than activations.
Thus, NUPES addresses the first limitation of PowerQuant: learning the weight values in a GPTQ fashion. Furthermore, it also allows to learn the power exponent parameter $a$ alongside the weight values.

\begin{figure}[!t]
    \centering
    \includegraphics[width = 0.75\linewidth]{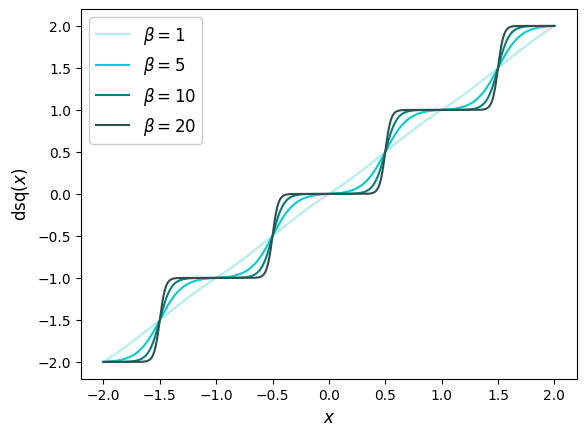}
    \caption{Illustration of the graph of the dsq transformation for different values of the steepness hyperparameter $\beta$.}
    \label{fig:dsq_graph}
\end{figure}

\subsection{Gradient Based Optimization of the Power Exponent}\label{sec:methodology_exponent}
Learning the power exponent $a$ through stochastic gradient descent requires the computation of the derivative 
\begin{equation}\label{eq:derivative_a}
    \frac{\partial X^a}{\partial a} = X^a log(X).
\end{equation}
This, however, raises a number of challenges. First, by definition of the power quantization operator (eq \ref{eq:power_operators}), this derivative appears in both the numerator and the input of the scaling function $s$. However, $s$ shall not be optimized through stochastic gradient descent as it can be analytically derived from the ranges of the tensor $X$. Consequently, we only compute the derivative for the numerator and derive the new scaling value (\textbf{update scale}) with the updated value of $a$, thus addressing this first issue. Second, the current formulation of the gradient is not numerically stable in practice, as near zero values in $X$ lead to infinite gradients due to the log function. This can be solved by clipping the values of $X$ (\textbf{num. stability}) while computing gradients. Third, for a given layer, we compute the update of the parameter $a$ from two transformations: the quantization of the weights (which depends on the batch size) and the quantization of the inputs. Consequently, a simple addition of these two update terms would lead to an unbalance and a strong dependency on the batch size. We alleviate this issue by averaging the gradients contributions (\textbf{balanced grads}), from the weight and input quantization, independently before their combination.

Algorithm \ref{alg:learn_a} summarizes these adaptations of the optimization process that are necessary to learn $a$ during any training process from GPTQ \cite{nagel2020up} to QAT \cite{zhou2017incremental}.  As a result, NUPES solves all the aforementioned shortcomings of GPTQ methods applied to PowerQuant, a claim that we experimentally validate further below.

\begin{algorithm}[!t]
\caption{Learn the exponent parameter $a$}\label{alg:learn_a}
\begin{algorithmic}
\Require a layer $f_l$ with weights $W_l\in\mathbb{R}^{N\times M}$ and inputs $X\in\mathbb{R}^{B\times N}$
\State $\text{scale}_x \leftarrow s(sign(X) \times |X|^a)$ \Comment{update scale (no $\nabla$)}
\State $\text{scale}_w \leftarrow s(sign(W_l) \times |W_l|^a)$ \Comment{update scale (no $\nabla$)}
\State ~
\State \textbf{Forward Pass}
\State $X \leftarrow \left\lfloor \frac{\text{sign}(X)\times |X|^a}{\text{scale}_x} \right\rceil$
\State $W_l \leftarrow \left\lfloor \frac{\text{sign}(W_l)\times |W_l|^a}{\text{scale}_w} \right\rceil$
\State $Y \leftarrow f_l(X,W_l)$
\State ~
\State \textbf{Backward Pass}
\State $X_{\text{clipped}} \leftarrow \text{clip}(|X|, 10^{-6}, \infty)$ \Comment{num stability}
\State $W_{\text{clipped}} \leftarrow \text{clip}(|W|, 10^{-6}, \infty)$ \Comment{num stability}
\State $\nabla_{\text{from inputs}} \leftarrow \frac{X_{\text{clipped}}^a log(X_{\text{clipped}})}{B\times N}$ 
\State $\nabla_{\text{from weights}} \leftarrow \frac{W_{\text{clipped}}^a log(W_{\text{clipped}})}{N\times M}$
\State $\nabla_a \leftarrow \nabla_{\text{from inputs}} + \nabla_{\text{from weights}}$ \Comment{balanced grads}
\end{algorithmic}
\end{algorithm}

\section{Experiments}

In this section, we evaluate NUPES on three main aspects:
\begin{itemize}
    \item the necessity to address the aforementioned shortcomings related to the optimization of the power exponent and quantized weight values,
    \item the accuracy improvement from NUPES over the data-free PowerQuant and other GPTQ methods such as AdaRound \cite{nagel2020up} or BrecQ \cite{li2021brecq},
    \item the interest of power quantization methods to tackle LLMs weight quantization without the use of unpractical mix of floating point and integer representation (contrary to e.g. QLoRa \cite{dettmers2023qlora} or OPTQ \cite{frantar2023optq}).
\end{itemize}

\begin{table}[!t]
    \centering
    \caption{Evaluation, in W4/A8, of PowerQuant with AdaRound in a naïve combination.}
    \label{tab:pq+adaround}
    \begin{tabular}{c|c|c|c}
    \hline
         & ResNet 50 & RetinaNet & ViT \\
    \hline
        PowerQuant & 74.892 & 21.562 & 80.354 \\
        AdaRound & 75.322 & 27.258 & 79.590 \\
        PowerQ + AdaR & 72.384 & 20.346 & 77.214 \\
    \hline
    \end{tabular}
\end{table}
\begin{figure*}
  \begin{minipage}[!t]{.30\linewidth}
    \centering
    \vspace{0.31in}
    \caption
      {
        Graph plots of the $\beta$ schedulers with respect to the number of optimization steps.
        \label{fig:scheduler}
      }
    \vspace{0.07in}
    \includegraphics[width = \linewidth]{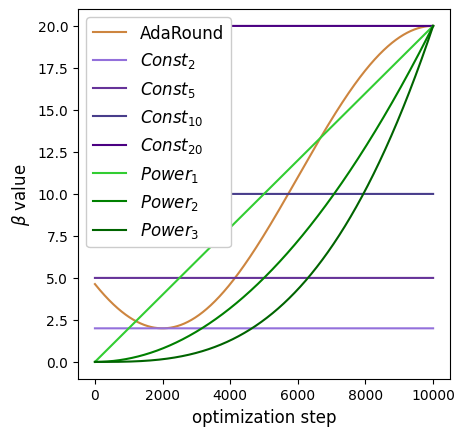}
    
  \end{minipage}
  \begin{minipage}[!t]{.69\linewidth}
    \centering
    \setlength\extrarowheight{2pt}
    \setlength\tabcolsep{0.5pt}
    \captionof{table}
      {%
        Evaluation of the different $\beta$ schedulers for W4/A4 quantization. We also provide the performance of the PowerQuant method as a reference (ref)
        \label{tab:scheduler}
      }
    \begin{tabular}{c|c|c|c|c|c|c|c|c|c}
     \hline
     model & AdaRound & $\text{Const}_{2}$ & $\text{Const}_{5}$ & $\text{Const}_{10}$ & $\text{Const}_{20}$ & $\text{Power}_{1}$ & $\text{Power}_{2}$ & $\text{Power}_{3}$ & ref\\
     \hline
     ResNet 18 & 17.366 & 19.840 & 41.534 & 54.500 & \textbf{64.528} & 14.600 & 24.640 & 29.134 & 56.386 \\
     ResNet 34 & 26.934 & 38.940 & 57.022 & 59.402 & \textbf{68.236} & 31.708 & 31.648 & 31.150 & 62.904 \\
     ResNet 50 & 6.254 & 14.578 & 28.496 & 56.320 & \textbf{68.758} & 8.276 & 14.034 & 8.408 & 62.142 \\
     ResNet 101 & 4.842 & 9.980 & 42.656 & 53.328 & \textbf{71.736} & 22.622 & 25.346 & 23.054 & 64.562 \\
     RetinaNet & 23.118 & 21.472 & 20.692 & 23.652 & \textbf{32.692} & 19.486 & 18.814 & 19.392 & 3.618 \\
     ViT b16 & 55.454 & 58.258 & 59.004 & 59.004 & \textbf{79.578} & 59.026 & 53.180 & 57.788 & 74.134 \\
     ViT l16 & 11.144 & 23.624 & 33.318 & 32.750 & \textbf{34.544} & 14.072 & 23.070 & 6.842 & 33.310 \\
     ViT h14 & 53.219 & 32.362 & 70.878 & 84.796 & \textbf{87.190} & 36.490 & 36.436 & 36.364 & 85.906 \\
     \hline
    \end{tabular}
  \end{minipage}
\end{figure*}

\subsection{Implementation and Datasets}
We conducted our evaluations on a wide range of ConvNets and Transformer networks. Regarding ConvNets, we considered the ResNet \cite{he2016deep} family, MobileNet v2 \cite{sandler2018mobilenetv2}, EfficientNet B0 \cite{tan2019efficientnet} for image classification on ImageNet \cite{imagenet_cvpr09} and RetinaNet \cite{lin2017focal} with a ResNet backbone for dense predictions (object detection) on COCO \cite{cocodataset}. Regarding Transformer architectures, we considered the ViT \cite{dosovitskiy2020image} family for image classification on ImageNet and OPT models \cite{zhang2022opt} for LLM evaluation on a wide set of common sense reasoning tasks \cite{clark2019boolq,bisk2020piqa,zellers2019hellaswag,sakaguchi2021winogrande,clark2018think,mihaylov2018can}. We also included results on the more recent Dolly v2 network \cite{dollyv2}.

In our experiments, the baseline performance of these models come from TensorFlow \cite{tensorflow2015-whitepaper} implementations for image classification, and from torchvision \cite{torchvision} for object detection. The LLMs implementation and weight values were downloaded from HuggingFace \cite{HuggingFace}. When reporting accuracy, we use the commonly adopted notation W4/A8 for a 4 bits quantization of the weights and 8 bits for the activations. In all our experiments, we follow the standard format \cite{nagel2019data,nagel2020up} and use W8/A8 representations for the first and last layers of convolutional neural networks.

NUPES as well as other quantization methods were implemented in both Torch and TensorFlow for experiments and run on a single A100 Nvidia GPU. We use our own implementations for SQuant, DFQ and PowerQuant while we report performance of other methods based upon SQuant paper \cite{squant2022}. In all our experiments with gradient descent optimization, we share the same set of hyper-parameter values. We use the Adam optimizer \cite{kingma2014adam} with base parameters, a batch-size of $32$, for $10k$ steps over $1024$ data samples from the training set and no learning scheduler. The scheduler choice for $\beta$ from equation \ref{eq:dsq}, is discussed in section \ref{sec:scheduler}.

In the next section, we conduct an ablation study to highlight the impact of each component in NUPES.

\subsection{Ablation study}

\subsubsection{GPTQ Loss Scheduler}\label{sec:scheduler}

In Table \ref{tab:pq+adaround}, we report results obtained using a naïve combination of PowerQuant and AdaRound. This motivates the need for the adaptations introduced in Section \ref{sec:methodology_kernel}: indeed, on top of the memory footprint reduction, we show that the integration of PowerQuant within AdaRound is not straightforward. Intuitively, this comes from the non-uniform quantized distribution with creates unbalance between weight values. This is solved by NUPES with allows for more flexibility on the values of $\epsilon$.

\begin{table*}[!t]
\caption{Comparison to state-of-the-art GPTQ techniques. We report the W4/A4 quantized accuracies across convolutional neural networks and transformers. The first set of quantization methods are data-free while the second and third sets leverage a calibration set.}
\label{tab:sota_gptq}
\centering
\setlength\extrarowheight{2pt}
\setlength\tabcolsep{3pt}
    \begin{tabular}{c|c|c|c|c|c|c|c|c}
     \hline
     method & ResNet 18 & ResNet 34 & ResNet 50 & MobileNet v2 & EfficientNet B0 & RetinaNet & ViT b16 & ViT h14 \\
     \hline
     full-precision & 69.674 & 73.230 & 76.150 & 72.074 & 77.618 & 37.294 & 80.978 & 88.434 \\
     \hline
     \hline
     DFQ \cite{nagel2019data} ICCV 2019 & 29.602 & 40.698 & 28.548 & 0.232 & 0.112 & 0.256 & 3.354 & 0.176 \\
     SQuant \cite{squant2022} ICLR 2022 & 48.126 & 49.100 & 52.042 & 0.398 & 0.104 & 0.191 & 3.280 & 0.100 \\
     SPIQ \cite{yvinec2022spiq} WACV 2023 & 50.257 & 52.517 & 52.752 & 0.572 & 3.623 & 0.382 & 4.007 & 0.514 \\
     PowerQuant \cite{yvinec2023powerquant} ICLR 2023 & 56.386 & 62.904 & 62.142 & 0.348 & 3.618 & 2.241 & 74.134 & 85.906 \\
     \hline
     \hline
     AdaRound \cite{nagel2020up} ICML 2020 & 60.258 & 65.174 & 61.656 & 8.840 & 0.102 & 21.392 & 29.906 & 23.070 \\
     BrecQ \cite{li2021brecq} ICLR 2021 & 28.650 & 40.050 & 63.782 & 38.230 & 0.110 & 18.922 & 22.228 & 25.686 \\
     QDrop \cite{wei2022qdrop} ICLR 2022 & 63.448 & 64.176 & 65.766 & 40.984 & 0.100 & 20.812 & - & - \\
     PDQuant \cite{liu2023pd} CVPR 2023 & 63.471 & 64.148 & 66.440 & 41.464 & 0.100 & 21.562 & - & - \\
     \hline
     \hline
     NUPES (learn $a$) & 57.524 & 62.378 & 62.468 & 5.902 & 15.241 & 22.424 & 74.552 & 86.600 \\
     NUPES (learn $W$) & 64.528 & 68.236 & 68.758 & 42.239 & 18.132 & 32.692 & 79.578 & 87.190 \\
     NUPES (learn $W$ \& $a$) & \textbf{65.876} & \textbf{69.954} & \textbf{70.684} & \textbf{42.386} & \textbf{45.902} & \textbf{33.078} & \textbf{80.100} & \textbf{87.204} \\
     \hline
    \end{tabular}
\end{table*}
We control this flexibility through the steepness of the soft rounding function from equation \ref{eq:dsq}, we defined several schedulers for the $\beta$ parameter. Our first candidate was introduced in AdaRound \cite{nagel2020up} to also control the steepness of the soft rounding they proposed which has been successfully leveraged in many subsequent works \cite{li2021brecq,wei2022qdrop,liu2023pd}. For the other candidates, we considered naive schedulers that have also been applied to, general training, learning rate scheduling. Formally, our schedulers are
\begin{equation}
    \begin{cases}
        \text{AdaRound}(s) = 20 + \frac{-18}{2} \left(1 + \cos\left(\frac{s}{S} \pi\right)\right)\\
        \text{Const}_c(s) = c\\
        \text{Power}_c(s) = 20 \left(\frac{s}{S} \right)^{c} \\
    \end{cases}
\end{equation}
In Figure \ref{fig:scheduler}, we illustrate the evolution of each candidate scheduler with respect to the optimization step $s$ out of $S$. The original GPTQ method \cite{nagel2020up} and the subsequent counterparts have been using a sophisticated scheduler for $\beta$. The intuition behind the value of $\beta$ is the cost for any value $v = n + \epsilon$, with $n\in\mathbb{N}$ and $\epsilon \in [0;1]$, to learn a final value $v^* = n-1$ or $v^*=n+2$. The smaller $\beta$, the easier it gets to drift outside of the $[0;1]$ interval.

In Table \ref{tab:scheduler}, we compare these schedulers on several architectures and tasks. We observe that the constant strategy with $c=20$ systemically outperforms the other candidates as well as the PowerQuant baseline (ref column). This highlights that although it is important to allow the weights to shift by more than a single quantization step, this phenomenon should be constrained to a few occurrences. From now on, the NUPES method will always use the constant scheduler with $\beta = 20$. The resulting optimization process requires less memory footprint, as we only store the weights and not the $\epsilon$ tensor. This is crucial, especially for transformer architectures where the weight values represent most of the memory footprint.

As a result, we can optimize one of the decisive factors of the performance of a power quantized model: the quantized weight values. In the next section, we evaluate the ability of NUPES to also optimize the remaining decisive factor: exponent parameter.

\subsubsection{Power Exponent Optimization Backward Pass}
In Table \ref{tab:exponent_opt}, we evaluate the impact of the proposed learning process (sec \ref{sec:methodology_exponent}) for the exponent parameter. We consider the challenging int4 quantization configuration (W4/A4) and initialize the power values at $0.5$. First, we can observe that, in the absence of numerical stability safety nets, the learning processes never completes (hence the "-" accuracy). More precisely, we observe that the learning crashes after 4 optimization steps on average due to the computation of the logarithm of some zero values. Second, regardless on the architecture (transformer/convnet) and task (classification/detection), balancing the contributions of the inputs and weights as well as separating the update of the scales from the stochastic gradient descent process systematically improve the final accuracy. Furthermore, our results on RetinaNet show that the most important step is the isolation of the scales update. This can be explained by two facts: first, the choice of the scaling factor can be analytically computed. Second, the scaling factor is set to reduce the error introduced by the rounding and clipping processes, however theses two steps are not considered during gradient descent. In conclusion, the learning process of the exponent parameter is non-trivial and can lead to significant improvements to the accuracy of the quantized model, provided it is done correctly. For example, on RetinaNet, the PowerQuant method only achieves a 2.241 MAP, while NUPES improves this score by 1006\% (see Table \ref{tab:sota_gptq}).

\begin{table}[!t]
\caption{Evaluation of the impact of learning the exponent (and not the weights) in W4/A4 quantization. We provide the full-precision original performance (accuracy or MAP) of the pre-trained model.}
\label{tab:exponent_opt}
\centering
\setlength\extrarowheight{2pt}
\setlength\tabcolsep{3pt}
    \begin{tabular}{c|c|c|c|c}
     \hline
     model & num. stability & balanced grads & update scale & Accuracy \\
     \hline
     \multirow{5}{*}{\shortstack{ResNet50\\(76.15)}} & \xmark & \xmark & \xmark & - \\
      & \cmark & \xmark & \xmark & 46,552 \\
      & \cmark & \cmark & \xmark & 58.950 \\
      & \cmark & \xmark & \cmark & 61.997 \\
      & \cmark & \cmark & \cmark & \textbf{62.468} \\
     \hline
     \multirow{5}{*}{\shortstack{RetinaNet\\(37.294)}} & \xmark & \xmark & \xmark & - \\
      & \cmark & \xmark & \xmark & 4.084 \\
      & \cmark & \cmark & \xmark & 4.128 \\
      & \cmark & \xmark & \cmark & 19.838 \\
      & \cmark & \cmark & \cmark & 22.424 \\
     \hline
     \multirow{5}{*}{\shortstack{ViT b16\\(78.05)}} & \xmark & \xmark & \xmark & - \\
      & \cmark & \xmark & \xmark & 39.282 \\
      & \cmark & \cmark & \xmark & 62.142 \\
      & \cmark & \xmark & \cmark & 73,794 \\
      & \cmark & \cmark & \cmark & \textbf{74.552} \\
     \hline
    \end{tabular}
\end{table}

In Figure \ref{fig:power_exponent}, we display the learned power exponent values with respect to the layer's depth. Although the result seems chaotic, we can still draw several conclusions. First, we observe that the first layer tends to require higher exponent values, this suggests that they are characterized by less peaky distributions. This behavior, can also be attributed to the larger bit-width (8 bits) for ResNet and RetinaNet.
Second, we observe that all layers fall in the range $a_l\in[0.213; 0.774]$ which is in adequacy with the great performance of the value $0.5$ that we leveraged in PowerQuant \cite{yvinec2023powerquant}. This result is consistent regardless of the initial value of $a$. Third, the most important conclusion, we solved the problem we faced in our initial method: we can learn distinct power exponent values per layer and improve the final accuracy.

\begin{figure}[!t]
    \centering
    \includegraphics[width = \linewidth]{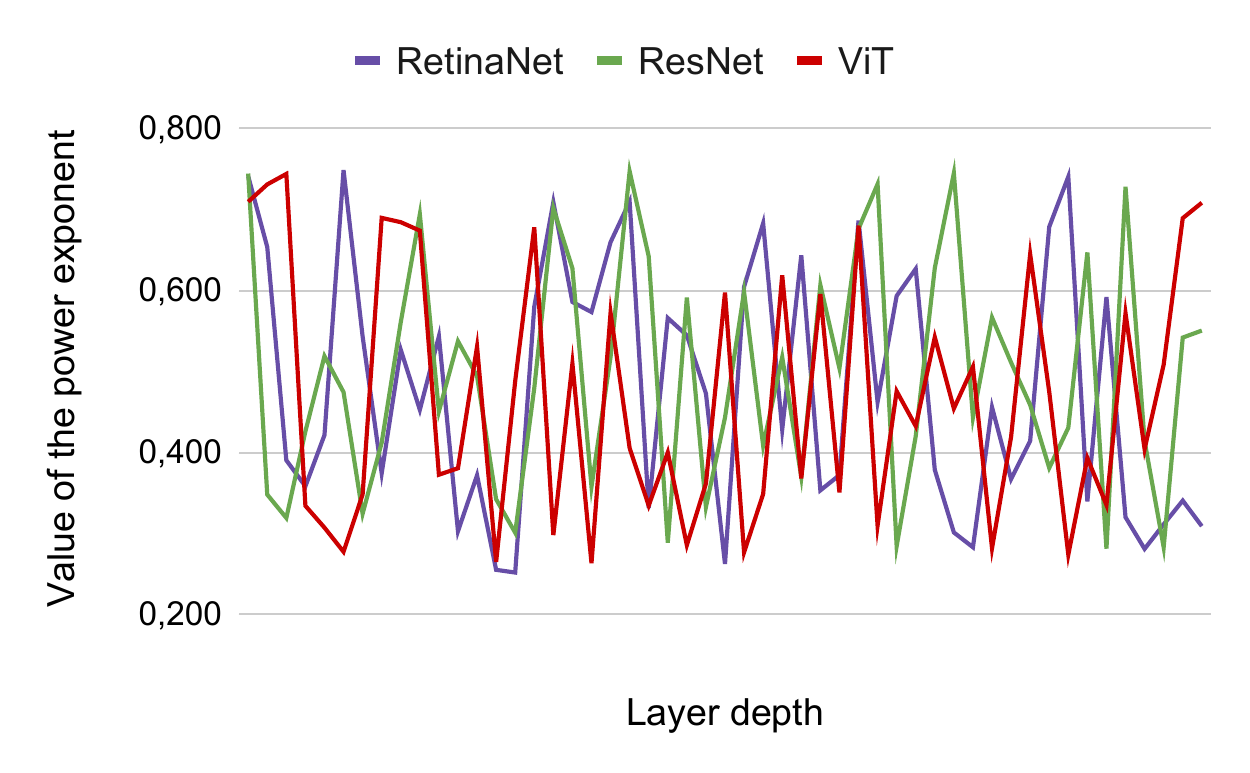}
    \caption{Plot of the learned power exponent values for several architectures in W4/A4.}
    \label{fig:power_exponent}
\end{figure}

In summary, we demonstrated the strong added value from each components of NUPES over the already strong PowerQuant baseline. In the following section, we show how PowerQuant and NUPES compare to other data-free and PTQ methods respectively.

\subsection{Main result: Comparison to other GPTQ methods}

\begin{table*}[!t]
\caption{Evaluation on data-free quantization methods on large language models W4/A16 quantization for common sense reasoning tasks.}
\label{tab:llms}
\centering
\setlength\extrarowheight{2pt}
\setlength\tabcolsep{3pt}
    \begin{tabular}{c|c|c|c|c|c|c|c|c|c}
     \hline
     model & method & OpenBookQA & ARC-E & ARC-C & WinoGrande & HellaSwag & PIQA & BoolQ & Average \\
     \hline
     \hline
     \multirow{4}{*}{Dolly v2 3B} & full-precision & 27.600 & 61.742 & 34.044 & 59.274 & 49.861 & 73.885 & 58.315 & 52.103 \\
     & DFQ \cite{nagel2019data} ICCV 2019 & 26.000 & 55.808 & 27.730 & 57.616 & 43.567 & 71.273 & 53.456 & 47.921 (-4.182) \\
     & SQuant \cite{squant2022} ICLR 2022 & 26.200 & 55.934 & 28.328 & 57.301 & 43.587 & 71.491 & 53.700 & 48.077 (-4.026) \\
     & PowerQuant \cite{yvinec2023powerquant} ICLR 2023 & \textbf{27.200} & \textbf{61.880} & \textbf{33.253} & \textbf{58.950} & \textbf{48.560} & \textbf{73.905} & \textbf{57.609} & \textbf{51.622} (-0.481) \\
     \hline
     \hline
     \multirow{4}{*}{Dolly v2 7B} & full-precision & 30.600 & 64.141 & 37.713 & 61.010 & 52.778 & 74.755 & 64.862 & 55.123 \\
     & DFQ \cite{nagel2019data} ICCV 2019 & 23.600 & 53.956 & 33.020 & 54.775 & 44.633 & 69.260 & 64.801 & 49.149 (-5.973) \\
     & SQuant \cite{squant2022} ICLR 2022 & 24.200 & 54.167 & 33.106 & 54.854 & 44.573 & 69.260 & 64.801 & 49.280 (-5.842) \\
     & PowerQuant \cite{yvinec2023powerquant} ICLR 2023 & \textbf{30.400} & \textbf{62.386} & \textbf{35.214} & \textbf{60.537} & \textbf{52.542} & \textbf{74.776} & \textbf{65.034} & \textbf{54.413} (-0.710) \\
     \hline
     \hline
     \multirow{4}{*}{OPT 13B} & full-precision & 27.000 & 61.953 & 33.020 & 65.746 & 52.390 & 76.714 & 64.954 & 54.540 \\
     & DFQ \cite{nagel2019data} ICCV 2019 & 25.400 & 59.975 & 29.836 & 63.062 & 49.044 & 75.734 & 49.786 & 50.405 (-4.135) \\
     & SQuant \cite{squant2022} ICLR 2022 & 25.547 & 60.145 & 29.911 & 63.097 & 49.032 & 75.734 & 49.786 & 50.465 (-4.075)\\
     & PowerQuant \cite{yvinec2023powerquant} ICLR 2023 & \textbf{27.000} & \textbf{61.816} & \textbf{32.741} & \textbf{64.088} & \textbf{51.115} & \textbf{76.354} & \textbf{67.217} & \textbf{54.333} (-0.207) \\
     \hline
     \hline
     \multirow{4}{*}{OPT 30B} & full-precision & 30.600 & 64.941 & 34.471 & 68.272 & 54.272 & 77.911 & 70.061 & 57.218 \\
     & DFQ \cite{nagel2019data} ICCV 2019 & 27.400 & 56.860 & 30.119 & 63.931 & 50.119 & 75.680 & 67.339 & 53.064 (-4.154) \\
     & SQuant \cite{squant2022} ICLR 2022 & 27.400 & 57.077 & 30.231 & 63.967 & 50.140 & 75.676 & 67.339 & 53.119 (-4.099) \\
     & PowerQuant \cite{yvinec2023powerquant} ICLR 2023 & \textbf{30.500} & \textbf{63.552} & \textbf{34.276} & \textbf{67.930} & \textbf{53.625} & \textbf{78.183} & \textbf{69.966} & \textbf{56.862} (-0.356) \\
     \hline
    \end{tabular}
\end{table*}
\begin{figure}[!t]
    \centering
    \includegraphics[width = \linewidth]{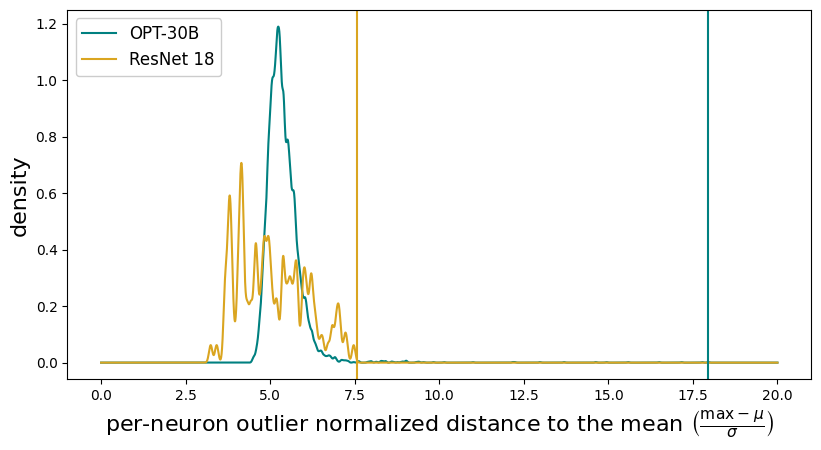}
    \caption{Distribution of outlier weight values defined by their distance to the mean value in terms of standard deviations for a ResNet 18 and OPT model. This highlights the specific challenge introduced by LLMs in terms of quantization ranges.}
    \label{fig:outliers}
\end{figure}

Among state-of-the-art data-free quantization techniques, PowerQuant stands out as the most effective technique. In this study, we removed the comparison to methods that require data generation such as GDFQ \cite{xu2020generative}, because such methods fail to scale to both very large models and are not suited for models that do not use batch normalization layers \cite{hatamizadeh2022gradvit}. For these reasons, we focused on fully data-free techniques that are effective at all model sizes.

In Table \ref{tab:sota_gptq}, we report our extensive study of post-training W4/A4 quantization techniques on convolutional neural networks (ResNets, MobileNets and EfficientNets) as well as transformers from ViT b16 (86M parameters) to ViT h14 (600M parameters). In this extreme compression regime, we observe the limits of previous state-of-the art methods SQuant \cite{squant2022} and SPIQ \cite{yvinec2022spiq}. This is not the case for PowerQuant which already achieves strong results on ResNets and transformers and, as such, offers a very strong baseline for the proposed NUPES method.

The proposed evaluation is particularly challenging as compared to results shared in SQuant \cite{squant2022} or PowerQuant \cite{yvinec2023powerquant} as we use symmetric quantization with no zero points. Still, PowerQuant achieves results on par with post-training methods such as AdaRound on ResNet architectures and even outperforms AdaRound on vision transformer quantization. Interestingly, our empirical study shows that extensions of AdaRound, such as BrecQ \cite{li2021brecq}, QDrop \cite{wei2022qdrop} and PD-Quant \cite{liu2023pd} offer strong improvements on MobileNet architectures but are less effective on other architectures. This can be attributed to two factors: first, we applied the same training procedure as in AdaRound \cite{nagel2020up} with 10K optimization steps (contrary to the 20K used in QDrop and PD-Quant). Second, their initialization corresponds to the DFQ baseline, which suffers from huge accuracy drops in W4/A4 and thus limits the ability of such methods to match NUPES, especially on vision transformers or at low bit-width formats.

The proposed NUPES method significantly and consistently outperforms state-of-the-art methods on all the benchmarked models and tasks in W4/A4 quantization. In particular, we observe that the largest contribution comes from the optimization of the weight values as the second last row is also the second best performance on each architecture. We also emphasize that the improvements over AdaRound introduced in other post-training techniques \cite{li2021brecq,wei2022qdrop,liu2023pd} could also be adapted to NUPES for further accuracy improvements.

As our results suggest, NUPES shows its highest scores on the transformer architecture. This can be, in part, attributed to the performance of the baseline PowerQuant method which is well suited for the weight and activation distributions learned by such architectures. In the following section, we study the empirical ability of PowerQuant to quantize the largest transformers \textit{i.e.} large language models.
In our evaluation, we propose two discussions: first, PowerQuant without in its current form, second, PowerQuant in combination with group-wise quantization.

\subsection{Quantization at all Sizes: handling outliers}
In Table \ref{tab:llms}, we report our results for several large language models on common sense reasoning tasks. We do not use group-wise quantization as it leads to incompatibility with activation quantization due to the constraint of dimensionality as explained in SPIQ \cite{yvinec2022spiq}. In other words, while we can demonstrate that group-wise quantization can lead to higher compression rate for the weights, such methods are bound to never quantize the activations. Consequently, we first highlight the fact that, to the best of our knowledge, PowerQuant bares the most promising results for the future of LLM inference. In fact, our results suggest that it preserves the performance of the full-precision model with up to 1 point drop using 4 bits quantization without any specific techniques to address the presence of outliers. This can be explained by the ability to stretch out the quantization step of larger values as illustrated in Figure \ref{fig:power_comparison} (sec \ref{sec:powerquant_methodo}).

\begin{table}[!t]
    \centering
    \setlength\extrarowheight{2pt}
    \setlength\tabcolsep{3pt}
    \caption{Evaluation of the proposed method with group-wise quantization on LLMs in W3/A16 and grouping of size 128. We report the average score on common sense reasoning tasks like in Table \ref{tab:llms}.}
    \label{tab:group_wise_quantization}
    \begin{tabular}{c|c|c|c|c}
    \hline 
        LLM & DFQ & OPTQ & PowerQ & PowerQ + Group-Wise \\
    \hline 
        OPT 13B & 35.210 & 52.643 & 46.448 & \textbf{54.472} \\
        OPT 30B & 32.184 & 51.132 & 47.758 & \textbf{56.377} \\
    \hline 
    \end{tabular}
\end{table}

In Table \ref{tab:group_wise_quantization}, we report the the influence of group-wise quantization \cite{park2022nuqmm} on post-training quantization for LLMs (and add DFQ as a reference). While some methods, like OPTQ \cite{frantar2023optq}, leverage re-ordering and grouping with success in terms of accuracy, we foresee that this will lead to future bottlenecks. Nonetheless, these methods can be combined with PowerQuant in a straightforward manner and immediately offer significant accuracy improvements for 3 bits quantization. We observe that PowerQuant with group-wise quantization unlocks near full-precsion performance (less than one percent drop) using only 3 bits. As a result, an OPT model with 13 billion parameters can be loaded on a device with only 5GB of memory rather than 26GB initially. 

\section{Conclusion}
In this work, we proposed a novel, non-uniform, post-training quantization technique. This method is based on the search for the best quantization operator among transformations that map multiplications to multiplications that we introduced in PowerQuant, a data-free quantization method. Because such transformations are power functions, the search for the best operator consists in the optimization of the power exponent. However, PowerQuant suffered from two limitations that we addressed in NUPES. First, we observed that PowerQuant could not be naively combined with gradient-based post-training quantization schemes such as AdaRound which are known for their strong results. NUPES alleviates this limitation by introducing differentiable soft quantization. Second, the optimization of the power exponent was initially derived from the weight values only and was not compatible with stochastic gradient descent. In NUPES, we solved this problem, enabling per-layer adaptation of PowerQuant and thus unlocking even higher accuracies. On top of significantly improving the performance of PowerQuant, NUPES also diminishes the memory footprint and computational requirements of GPTQ methods by using a single weight tensor and a single loss term instead of two. As a result, our empirical benchmark on ConvNets and transformers highlight that NUPES vastly outperforms other post-training compression techniques. In particular, on transformers, PowerQuant alone already shows remarkable result, in particular, at handling outliers. Consequently, we evaluated it on large language models. Our empirical results show that this non-uniform quantization scheme unlocks 4 bits and 3 bits weight encoding without damaging the accuracy nor requiring re-ordering and grouping which are incompatible with activation quantization.

\begin{IEEEbiography}[{\includegraphics[width=1in,height=1.25in,clip,keepaspectratio]{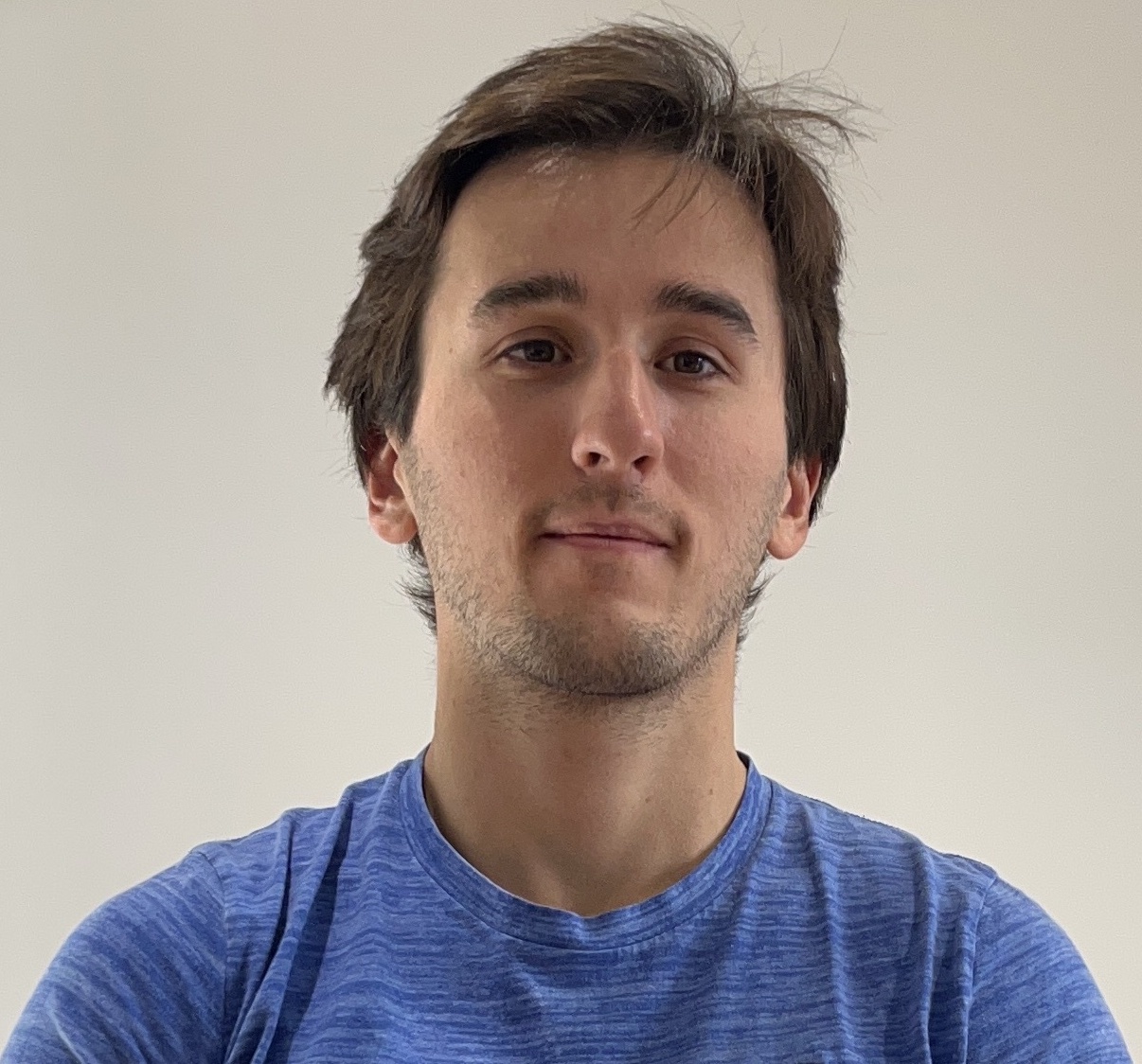}}]{Edouard YVINEC}
received his master's degree from Ecole Normale Superieure Paris-Saclay in 2020 and is currently a P.h.D. student at ISIR in Sorbonne Université. His research interests include but are not limited to DNN solutions for computer vision tasks, compression and acceleration of such models.
\end{IEEEbiography}
\vspace{-0.5in}
\begin{IEEEbiography}[{\includegraphics[width=1in,height=1.25in,clip,keepaspectratio]{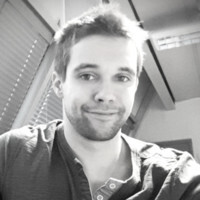}}]{Arnaud DAPOGNY} is a computer vision researcher at Datakalab. He graduated from Sup\'elec in 2011, Sorbonne University in 2013, and obtained a PhD at Institute for Intelligent Systems and Robotics (ISIR) in 2016. He worked as a post-doctoral fellow at LIP6. His works concern deep learning for computer vision and its application to automatic facial behavior as well as gesture analysis.
\end{IEEEbiography}
\vspace{-0.5in}
\begin{IEEEbiography}[{\includegraphics[width=1in,height=1.25in,clip,keepaspectratio]{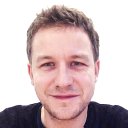}}]{Kevin BAILLY}
is associate professor with the Institute of Intelligent Systems and Robotics (ISIR) at Sorbonne University and Head of Research of Datakalab. He received the PhD degree in computer science from the Pierre et Marie Curie University in 2010 and was a postdoctoral researcher at Telecom Paris from 2010 to 2011. His research interests are in machine learning and computer vision applied to face processing and behavior analysis.
\end{IEEEbiography}

{
% \small
\bibliographystyle{ieee_fullname}
\bibliography{main.bib}
}
\newpage

%%%%%%%%%%%%%%%%%%%%%%%%%%%%%%%%%%%%%%%%%%%%%%%%%%%%%%%%%%%%%%%%%%%%%%%%%%%%%%%%%%%%%
%
%
%                                    Appendices
%
%
%%%%%%%%%%%%%%%%%%%%%%%%%%%%%%%%%%%%%%%%%%%%%%%%%%%%%%%%%%%%%%%%%%%%%%%%%%%%%%%%%%%%%
\begin{appendices}

\section{Discussion around the set of Automorphisms of \boldmath\texorpdfstring{$(\mathbb{R}_+^*,\times)$}{(R+,x)}}\label{sec:appendix_automorphism_set}
In this section, we provide a simple proof on the definition of the set of Automorphisms of $(\mathbb{R}_+^*,\times)$.

\begin{lemma}\label{thm:lemma_automorphism}
The set of continuous automorphisms of $(R_+^*,\times)$ is defined by the set of power functions $\mathcal{Q} = \{Q : x \mapsto x^a | a \in \mathbb{R}\}$.
\end{lemma}
\begin{proof}
We have that $\forall x\in \mathbb{R}_+, Q(x)\times Q(0) = Q(0)$ and $ \forall x\in \mathbb{R}_+, Q(x)\times Q(1) = Q(x)$
which induces that $Q$ is either the constant $1$ or $Q(0)=0$ and $Q(1)=1$. Because $Q$ is an automorphism we can eliminate the first option. Now, we will demonstrate that $Q$ is necessarily a power function. Let $n$ be an integer, then 
\begin{equation}
    Q(x^n)=Q(x)\times Q(x^{n-1})=Q(x)^2\times Q(x^{n-2})=\dots= Q(x)^n.
\end{equation}
Similarly, for fractions, we get $Q(x^{\frac{1}{n}}) \times \dots \times Q(x^{\frac{1}{n}}) = Q(x) \Leftrightarrow Q(x^{\frac{1}{n}})=Q(x)^{\frac{1}{n}}$.
Assuming $Q$ is continuous, we deduce that for any rational $a\in\mathbb{R}$, we have
\begin{equation}\label{eq:power_func_property}
    Q(x^a) = Q(x)^a
\end{equation}
In order to verify that the solution is limited to power functions, we use a \textit{reductio ad absurdum}.
Assume $Q$ is not a power function. Therefore, there exists $(x,y)\in\mathbb{R}_+^2$ and $a \in \mathbb{R}$ such that $Q(x)\neq x^a$ and $Q(y)=y^a$. By definition of the logarithm, there exists $b$ such that $x^b=y$. We get the following contradiction, from \eqref{eq:power_func_property},
\begin{equation}
\begin{cases}
    Q({x^b}^a) = Q(y^a) = y^a\\
    Q({x^b}^a) = Q({x^a}^b) = {Q(x^a)}^b \neq \left({x^a}^b = y^a\right)
\end{cases}
\end{equation}
Consequently, the suited functions $Q$ are limited to power functions \textit{i.e.} $\mathcal{Q} = \{Q : x \mapsto x^a | a \in \mathbb{R}\}$.
\end{proof}

The definition of the other automorphisms of $(\mathbb{R}_+^*,\times)$ requires the axiom of choice \cite{herrlich2006axiom}. Such automorphisms are not applicable in our case as we work on a finite subset of $\mathbb{R}$ defined by floating values.

\section{Convexity and Uniqueness}\label{sec:appendix_maths}
\subsection{Local Convexity}\label{sec:appendix_convexity}
We provide the main steps of the proof, from PowerQuant \cite{yvinec2023powerquant}, of the local convexity of the minimization problem from equation \ref{eq:minimization_problem}. 
\begin{lemma}
The minimization problem defined as
\begin{equation}
    \arg\min_a\left\{ \left\| x - Q^{-1}_a\left(Q_a(x)\right) \right\|_p\right\}
\end{equation}
is locally convex around any solution $a^*$.
\end{lemma}
\begin{proof}
We recall that $\frac{\partial x^a}{\partial a} = x^a\log(x)$. The function $\left\| x - Q^{-1}_a\left(Q_a(x)\right)\right\|$ is differentiable. We assume $x$ positive without loss of generality and note $y = \max |x|$, then
\begin{equation}
    \frac{\partial Q^{-1}_a\left(Q_a(x)\right)}{\partial a} = \frac{\partial \left|\left\lfloor (2^{b-1}-1)\frac{x^a}{y^a} \right\rfloor\frac{y^a}{2^{b-1}-1}\right|^{\frac{1}{a}}}{\partial a}.
\end{equation}
with $B=2^{b-1}-1$. By using the standard differentiation rules, we get,
\begin{equation}
    \frac{\partial Q^{-1}_a\left(Q_a(x)\right)}{\partial a} = -a^2y\left(\frac{\left\lfloor B \left(\frac{x}{y}\right)^a \right\rfloor}{B}\right)^{\frac{1}{a}}\log\left(\frac{\left\lfloor B \left(\frac{x}{y}\right)^a \right\rfloor}{B}\right).
\end{equation}
Now we can compute the second derivative of $Q^{-1}_a\left(Q_a(x)\right)$, and get
\begin{equation}
\begin{cases}
    T_1& = \frac{1-p}{p^2}\frac{|x_i - Q^{-1}_a(Q_a(x_i)|^{\frac{1}{p}}}{(x_i - Q^{-1}_a(Q_a(x_i))^2}\left(\frac{\partial Q^{-1}_a\left(Q_a(x)\right)}{\partial a}\right)^2\\
    T_2 &=\frac{(x_i - Q^{-1}_a(Q_a(x_i))|x_i - Q^{-1}_a(Q_a(x_i)|^{\frac{1}{p}-2}}{p}\\
    T_3&=\frac{\partial^2 Q^{-1}_a\left(Q_a(x)\right)}{\partial a^2}\\
    \frac{\partial^2 \left| x - Q^{-1}_a\left(Q_a(x)\right)\right|}{\partial a^2} &= T_1 + T_2\times T_3\\
\end{cases}
\end{equation}
We know that $T_1>0$ and $T_3>0$, consequently, and $T_2$ is continuous in $a$. At $a^*$ the terms with $| x_i - Q^{-1}_a\left(Q_a(x_i)\right)|$ are negligible in comparison with $\frac{\partial^2 Q^{-1}_a\left(Q_a(x)\right)}{\partial a^2}$ and $\left(\frac{\partial Q^{-1}_a\left(Q_a(x)\right)}{\partial a}\right)^2$. Consequently, there exists an open set around $a^*$ where $T_1 > |T_2|T_3$, and $\frac{\partial^2 \left| x_i - Q^{-1}_a\left(Q_a(x_i)\right)\right|}{\partial a^2} > 0$. This concludes the proof.
\end{proof}

\subsection{Uniqueness of the Solution}\label{sec:appendix_uniqueness}
In this section we provide a summary of the proof on the uniqueness of the solution of the minimization of the quantization reconstruction error from PowerQuant \cite{yvinec2023powerquant}.
\begin{lemma}
The minimization problem over $x\in\mathbb{R}^{N}$ defined as
\begin{equation}
    \arg\min_a\left\{ \left\| x - Q^{-1}_a\left(Q_a(x)\right) \right\|_p\right\}
\end{equation}
has almost surely a unique global minimum $a^*$.
\end{lemma}
\begin{proof}
We assume that $x$ cannot be exactly quantized, \textit{i.e.} $\min_a\left\{ \left\| x - Q^{-1}_a\left(Q_a(x)\right) \right\|_p\right\} > 0$ which is true almost everywhere.
We use a \textit{reductio ad absurdum} and assume that there exist two optimal solutions $a_1$ and $a_2$ to the optimization problem. We expand the expression $\left\| x - Q^{-1}_a\left(Q_a(x)\right) \right\|_p$ and note the rounding term $R_a$, we get
\begin{equation}
    \left\| x - Q^{-1}_a\left(Q_a(x)\right) \right\|_p=\left\| x - \left|R_a\frac{\max{|x|^a}}{2^{b-1}-1}\right|^{\frac{1}{a}}\text{sign}(x) \right\|_p.
\end{equation}
\begin{itemize}
    \item Assume $R_{a_1} = R_{a_2} = R$, the minimization problem is convex and has a unique solution, thus $a_1=a_2$. 
    \item Assume $R_{a_1} \neq R_{a_2}$, we note $D(R_a)$ the set of values $a$ such that the rounding function outputs the same value. Then there exists a value $a^* = ta_1 + (1-t)a_2$ with $t\in[0;1]$ such that $a^*$ has a strictly lower quantization error than $a_1$ and $a_2$ which is absurd
\end{itemize}
\end{proof}

\end{appendices}

\ifCLASSOPTIONcaptionsoff
  \newpage
\fi

\end{document}